\documentclass[a4paper,12pt]{article}
\usepackage{tabularx,graphicx,moreverb}
\usepackage{amsmath,amssymb,amsthm}
\usepackage{bm}
\usepackage{ascmac}
\usepackage{makeidx}
\usepackage[english]{babel}
\usepackage{wrapfig}
\usepackage{color}
\usepackage{algorithm}
\usepackage{algorithmic}
\allowdisplaybreaks[1]
\usepackage{natbib}
\usepackage{here}
\usepackage[margin=1.0in]{geometry}
\usepackage{authblk}
\usepackage{enumitem}

\newcommand{\abovespace}{\,}
\newcommand{\belowspace}{l\,}

\theoremstyle{break}
\newtheorem{lem}{Lemma}
\newtheorem{them}{Theorem}
\newtheorem{asmp}{Assumption}

\newtheorem{defi}{Definition}
\newtheorem{cor}{Corollary}
\newcommand{\argmax}{\operatornamewithlimits{argmax}}
\title{Doubly Decomposing\\Nonparametric Tensor Regression}

\author[1]{Masaaki Imaizumi}
\author[2]{Kohei Hayashi}
\affil[1]{University of Tokyo}
\affil[2]{National Institute of Informatics}
\date{}
\begin{document}
\thispagestyle{empty}

\maketitle

\setcounter{page}{1}



\begin{abstract}
  Nonparametric extension of tensor regression is proposed.
  Nonlinearity in a high-dimensional tensor space is broken into
  simple local functions by incorporating low-rank tensor
  decomposition. Compared to naive nonparametric approaches, our
  formulation considerably improves the convergence rate of estimation
  while maintaining consistency with the same function class under
  specific conditions. To estimate local functions, we develop a
  Bayesian estimator with the Gaussian process prior. Experimental
  results show its theoretical properties and high performance in
  terms of predicting a summary statistic of a real complex network.
\end{abstract}


\section{Introduction}


\emph{Tensor regression} deals with matrices or tensors (i.e.,
multi-dimensional arrays) as covariates (inputs) to predict scalar
responses
(outputs)~\cite{wang2014,hung2013,zhao2014,zhou2013,tomioka2007,suzuki2015,guhaniyogi2015}. 
Suppose we have a set of $n$ observations $D_n =
\{(Y_i,X_i)\}_{i=1}^n$; $Y_i \in \mathcal{Y}$ is a respondent variable
in the space $\mathcal{Y} \subset \mathbb{R}$ and $X_i \in
\mathcal{X}$ is a covariate with $K$th-order tensor form in the space
$\mathcal{X} \subset \mathbb{R}^{I_1 \times \ldots \times I_K}$, where
$I_k$ is the dimensionality of order $k$. With the above setting, we
consider the regression problem of learning a function $f :
\mathcal{X} \rightarrow \mathcal{Y}$ as
\begin{align}\label{eq:TR}
	Y_i = f(X_i) + u_i,
\end{align}
where $u_i$ is zero-mean Gaussian noise with variance $\sigma^2$.
Such problems can be found in several applications. For example,
studies on brain-computer interfaces attempt to predict human intentions
(e.g., determining whether a subject imagines finger tapping) from
brain activities. Electroencephalography (EEG) measures brain activities as electric signals at several points (channels) on the scalp,
giving \textit{channel $\times$ time} matrices as
covariates. Functional magnetic resonance imaging captures blood flow
in the brain as three-dimensional voxels, giving \textit{X-axis
  $\times$ Y-axis $\times$ Z-axis $\times$ time} tensors.

There are primarily two approaches to the tensor regression problem. One is assuming
linearity to $f$ as
\begin{align}
	f(X_i) = \langle B,X_i \rangle, \label{lin_model}
\end{align}
where $B \in \mathbb{R}^{I_1 \times \ldots \times I_K}$ is a weight
parameter with the same dimensionalities as $X$ and $\langle B,X
\rangle = \sum_{j_1,\ldots,j_K = 1} ^{I_1,\ldots,I_K} B_{j_1 \ldots
  j_K} X_{j_1 \ldots j_K}$ denotes the inner product. Since $B$ is
very high-dimensional in general, several authors have incorporated a
low-rank structure to
$B$~\cite{dyrholm2007,zhou2013,hung2013,wang2014,suzuki2015,guhaniyogi2015}. We
collectively refer to the linear models~\eqref{lin_model} with
low-rank $B$ as \emph{tensor linear regression (TLR)}.
As an alternative, a nonparametric approach has been
proposed~\cite{zhao2013,hou2015a}. When $f(X)$ belongs to a proper
functional space, with an appropriately choosing kernel function, the nonparametric method can
estimate $f$ perfectly even if $f$ is nonlinear.

In terms of both theoretical and practical aspects, the \emph{bias-variance tradeoff} is a central
issue. In TLR, the function class that the model can represent is
critically restricted due to its linearity and the low-rank
constraint, implying that the variance error is low but the bias error
is high if the true function is either nonlinear or full rank.
In contrast, the nonparametric method can represent a wide range of
functions and the bias error can be close to zero. However, at the
expense of the flexibility, the variance error will be high due to the
high dimensionality, the notorious nature of tensors.
Generally, the optimal convergence rate of nonparametric models
is given by 
\begin{align}
	O(n^{-\beta/(2\beta + d)}), \label{general-rate}
\end{align}
which is dominated by the input dimensionality $d$ and the smoothness of the true function $\beta$~\citep{tsybakov2008}. For
tensor regression, $d$ is the total number of $X$'s elements, i.e.,
$\prod_k I_k$. When each dimensionality is roughly the same as $
I_1\simeq\dots\simeq I_K$, $d=O(I_1^K)$, which significantly worsens the rate, and hinders application to even moderate-sized problems.
%
%

In this paper, to overcome the curse of dimensionality, we propose
\textit{additive-multiplicative nonparametric regression (AMNR)}, a
new class of nonparametric tensor regression. Intuitively, AMNR
constructs $f$ as the sum of local functions taking the component of a
rank-one tensor as inputs. In this approach, functional space and the input space are concurrently decomposed.
This ``double decomposition'' simultaneously reduces model complexity and the effect of noise.
For estimation, we propose a Bayes estimator with the Gaussian Process (GP) prior.
The following theoretical results highlight the desirable properties of AMNR.
Under some conditions,
\begin{itemize}[topsep=0pt,itemsep=-1ex,partopsep=1ex,parsep=1ex]
\item AMNR represents the same function class as the general
  nonparametric model, while
\item the convergence rate~\eqref{general-rate} is improved as $d=I_{k'}$ 
  ($k'=\argmax_kI_k$), which is $\prod_{k\not=k'}I_k$ times better.
\end{itemize}
We verify the theoretical convergence rate by simulation and
demonstrate the empirical performance for real application in network
science.

\section{AMNR: Additive-Multiplicative Nonparametric Regression}\label{sec:AMNR}

First, we introduce the basic notion of tensor
decomposition.
With a finite positive integer $R^*$, the
\emph{CANDECOMP/PARAFAC (CP) decomposition}~\cite{harshman1970,carroll1970}
of $X \in \mathcal{X}$ is defined as
\begin{align}
	X = \sum_{r=1}^{R^*} \lambda_r x_{r}^{(1)} \otimes x_{r}^{(2)} \otimes \ldots \otimes x_{r}^{(K)}, \label{CP_decomp}
\end{align}
where $\otimes$ denotes the tensor product, $x_{r}^{(k)} \in
\mathcal{X}^{(k)}$ is a unit vector in a set $\mathcal{X}^{(k)} :=
\{v | v \in \mathbb{R}^{I_k}, \|v\|=1\}$, and $\lambda_r$ is the scale of
$\{x_{r}^{(1)},\dots,x_{r}^{(K)}\}$ satisfying $\lambda_r \geq
\lambda_{r'}$ for all $r > r'$. In this paper, $R^*$ is the
rank of $X$.
%

A similar relation holds for functions. Here,
$\mathcal{W}^{\beta}(\mathcal{X})$ denotes a Sobolev space, which is $\beta$
times differentiable functions with support $\mathcal{X}$.
Let $g \in \mathcal{W}^\beta(S)$ be such a function. If $S$ is given
by the direct product of multiple supports as $S=S_1\times\dots\times
S_J$, there exists a (possibly infinite) set of local functions
$\{g_m^{(j)} \in \mathcal{W}^{\beta}(S_j)\}_{m}$ satisfying
\begin{align}\label{eq:sum-pro-g}
  g = \sum_m \prod_j g_m^{(j)}
\end{align}
for any $g$~\citep[Example 4.40]{hackbusch2012}.
This relation can be seen as an extension of tensor decomposition with infinite
dimensionalities.

\subsection{The Model}

For brevity, we start with the case wherein $X$ is rank one. Let $X = \bigotimes_k x_k := x_1\otimes\dots\otimes x_K$ with vectors $\{x_k \in \mathcal{X}^{(k)}\}_{k=1}^K$ and $f\in\mathcal{W}^\beta(\bigotimes_k \mathcal{X}^{(k)})$
be a function on a rank one tensor. For any $f$, we can construct
$\tilde{f}(x_1,\ldots,x_K) \in \mathcal{W}^{\beta}(
\mathcal{X}^{(1)} \times \ldots \times \mathcal{X}^{(k)})$ such that $\tilde{f}(x_1,\ldots,x_K) = f(X)$ using function composition as
$\tilde{f}=f\circ h$ with $h:(x_1,\dots,x_K)\mapsto\bigotimes_k x_K$.
Then, using \eqref{eq:sum-pro-g}, $f$ is decomposed into a set of local
functions $\{f_m^{(k)} \in \mathcal{W}^{\beta}(\mathcal{X}^{(k)})\}_m$ as:
\begin{align}
	f(X) = \tilde{f} (x_1,\ldots,x_K)= \sum_{m=1}^{M^*} \prod_{k=1}^K f_m^{(k)}(x^{(k)}), \label{the_model_rank1}
\end{align}
where $M^*$ represents the complexity of $f$ (i.e., the ``rank'' of
the model).

With CP decomposition, \eqref{the_model_rank1} is amenable to extend for $X \in \mathcal{X}$ having a higher rank.  For $R^* \geq 1$, we define AMNR as
follows:
\begin{align}
	f^{\mathrm{AMNR}}(X) := \sum_{m=1}^{M^*} \sum_{r=1}^{R^*} \lambda_{r} \prod_{k=1}^K f_m^{(k)}(x_{r}^{(k)}). \label{the_model}
\end{align}
Aside from the summation with respect to $m$, AMNR~\eqref{the_model}
is very similar to CP decomposition~\eqref{CP_decomp} in terms of
that it takes summation over ranks and multiplication over orders. In
addition, as $\lambda_r$ indicates the importance of component $r$ in CP decomposition, it controls how component $r$ contributes to the final output in AMNR.
Note that, for $R^* > 1$, equality between $f^{\mathrm{AMNR}}$ and $f \in
\mathcal{W}^{\beta}(\mathcal{X})$ does not hold in general; see Section~\ref{sec:theory}.


\section{Truncated GP Estimator}

\subsection{Truncation of $M^*$ and $R^*$}
To construct AMNR \eqref{the_model}, we must know $M^*$. However, this is unrealistic because we do not know the true function. More crucially, $M^*$ can be infinite, and in such a case the exact estimation is computationally infeasible. We avoid these problems using predefined $M<\infty$ rather than $M^*$ and ignore the contribution from $\{f_m^{(k)} : m>M\}$. This may increase the model bias; however, it decreases the variance of estimation. We discuss how to determine $M$ in Section~\ref{theory_est_without}.

For $R^*$, we adopt the same strategy as $M^*$, i.e., we prepare some
$R<R^*$ and approximate $X$ as a rank- $R$ tensor. Because this
approximation reduces some information in X, the prediction
performance may degrade.  However, if $R$ is not too small, this
preprocessing is justifiable for the following reasons. First, this
approximation possibly removes the noise in $X$. In real data such as
EEG data, $X$ often includes observation noise that hinders the
prediction performance. However, if the power of the noise is
sufficiently small, the low-rank approximation discards the noise as
the residual and enhances the robustness of the model. In addition,
even if the approximation discards some intrinsic information of $X$,
its negative effects could be limited because $\lambda$s of the
discarded components are also small.

\subsection{Estimation method and algorithm}

For each local function $f_m^{(k)}$, consider the GP prior
$GP(f_{m}^{(k)})$, which is represented as multivariate Gaussian
distribution $\mathcal{N}(0_{Rn}, K_m^{(k)})$ where $0_{Rn}$ is the
zero element vector of size $Rn$ and $K_m^{(k)}$ is a kernel Gram
matrix of size $Rn \times Rn$. The prior distribution of the local functions
$\mathfrak{F}:= \{f_{m}^{(k)}\}_{m,k}$ is then given by:
\begin{align*}
	\pi(\mathfrak{F}) = \prod_{m=1}^M \prod_{k=1}^K GP(f_{m}^{(k)}).
\end{align*}
From the prior $\pi(\mathfrak{F})$ and the likelihood $\prod_i
N(Y_i|f(X_i),\sigma^2)$, Bayes' rule yields the posterior
distribution:
\begin{align}
	&\pi(\mathfrak{F}|D_n) \notag \\
	&= \frac{\exp(-\sum_{i=1}^n (Y_i - G[\mathfrak{F}](X_i))^2/\sigma)}{\int \exp(-\sum_{i=1}^n (Y_i - G[\tilde{\mathfrak{F}}](X_i))^2/\sigma) \pi(\tilde{\mathfrak{F}}) d\tilde{\mathfrak{F}}} \pi(\mathfrak{F}), \label{GP_post} 
\end{align}
where $G[\mathfrak{F}](X_i) =\sum_{m=1}^{M}\sum_{r=1}^{R} \lambda_{r,i} \prod_{k=1}^K f_m^{(k)}(x_{r,i}^{(k)})$.
$\tilde{\mathfrak{F}} = \{\tilde{f}_{m}^{(k)}\}_{m,k}$ are dummy variables for the integral.
We use the posterior mean as the Bayesian estimator of AMNR:
\begin{align}
  \hat{f}_n = \int \sum_{m=1}^{M}\sum_{r=1}^{R} \lambda_{r,i}
  \prod_{k=1}^K f_m^{(k)} d\pi(\mathfrak{F}|D_n)
  d\mathfrak{F}. \label{GP_est}
\end{align}

To obtain predictions with new inputs, we derive the mean of the
predictive distribution in a similar manner.

Since the integrals in the above derivations have no analytical
solution, we compute them numerically by sampling. The details of the
entire procedure are summarized as follows. Note that $Q$ denotes the
number of random samples.
\begin{itemize}
\item \textbf{Step 1: CP decomposition of input tensors}\\
  With the dataset $D_n$, apply rank-$R$ CP decomposition to $X_i$
  and obtain $\{\lambda_{r,i}\}$ and $\{x_{r,i}^{(k)}\}$ for $i =
  1,\ldots,n$.
\item \textbf{Step 2: Construction of the GP prior distribution $\pi(\mathfrak{F})$}\\
  Construct a kernel Gram matrix $K_m^{(k)}$ from $\{x_r^{(k)}\}$
  for each $m$ and $k$, and obtain random samples of the multivariate Gaussian
  distribution $\mathcal{N}(0_{Rn}, K_m^{(k)})$. For each sampling $q
  = 1,\ldots,Q$, obtain a value $f_m^{(k)}(x_{r,i}^{(k)})$ for each
  $r,m,k$, and $i = 1,\ldots,n$. 
\item \textbf{Step 3: Computation of likelihood}\\
  To obtain the likelihood, calculate $\sum_m \sum_r \lambda_r \prod_k
  f_m^{(k)}(x_{r,i}^{(k)})$ for each sampling $q$ and obtain the
  distribution by (\ref{GP_post}). Obtain the Bayesian
  estimator $\hat{f}$ and select the hyperparameters (optional).
\item \textbf{Step 4: Prediction with the predictive distribution}\\
  Given a new input $X'$, compute CP decomposition and obtain
  ${\lambda'}_r$ and $\{{x'}_{r}^{(k)}\}_{r,k}$. Then, sample
  $f_m^{(k)}({x'}_{r}^{(k)})$ from the prior for each $q$. 
  By multiplying the likelihood calculated in Step 3, derive the
  predictive distribution of $\sum_m \sum_r \lambda_r \prod_k
  f_m^{(k)}({x'}_r^{(k)})$ and obtain its expectation with respect to
  $q$.
\end{itemize}

\paragraph{Remark}
Although CP decomposition is not unique up to sign
permutation, our model estimation is not affected by this. 
For example, tensor $X$ with $R^* = 1$ and $K=3$ has two
equivalent decompositions: (A) $x_1\otimes x_2 \otimes x_3$ and (B) $x_1 \otimes
(- x_2)\otimes (- x_3)$. If training data only contains pattern (A),
prediction for pattern (B) does not make sense. However, such a case is
pathological. Indeed, if necessary, we
can completely avoid the problem by flipping the sign of $x_1,x_2$, and $x_3$
at random while maintaining the original sign of $X$. Although the sign
flipping decreases the effective sample size, it is absorbed as a constant term and the convergence rate is
not affected.






\section{Theoretical Analysis} \label{sec:theory}

Our main interest here is the asymptotic behavior of distance between the true function that generates data and an estimator \eqref{GP_est}. Preliminarily, let $f^0\in\mathcal{W}^{\beta}(\mathcal{X})$ be the true function
and $\hat{f}_n$ be the estimator of $f^0$.
To analyze the distance in more depth, we introduce the notion of \emph{rank additivity}\footnote{This type of additivity is often assumed in multivariate and additive model analysis \cite{hastie1990,ravikumar2009}.} for functions, which is assumed implicitly when we extend \eqref{the_model_rank1} to \eqref{the_model}.

\begin{defi}[Rank Additivity]\label{asmp:add_sep}
	A function $f : \mathcal{X} \rightarrow \mathcal{Y}$ is rank additive if
	\begin{align*}
		 f\left( \sum_{r=1}^{R^*} \bar{x}_r \right) = \sum_{r=1}^{R^*} f ( \bar{x}_r),
	\end{align*}
        where $\bar{x}_r := \lambda_r x_{r}^{(1)} \otimes \ldots \otimes x_{r}^{(K)}$.
\end{defi}
Letting $f^*$ be a projection of $f^0$ onto the Sobolev space $f \in \mathcal{W}^{\beta}$ satisfying rank additivity, the distance is bounded above as 
\begin{align}
	\|f^0 - \hat{f}_n\| \leq \|f^0 - f^*\| + \|f^* - \hat{f}_n\|. \label{decomp_bound}
\end{align}
Unfortunately, the first term $\|f^0-f^*\|$ is difficult to evaluate, aside from a few exceptions; if $R^*>0$ or $f^0$ is rank additive, $\|f^0-f^*\|=0$.

Therefore, we focus on the rest term $\|f^*-\hat{f}_n\|$. By definition, $f^*$ is rank additive and the functional tensor decomposition \eqref{the_model_rank1} guarantees that $f^*$ is decomposed as the AMNR form \eqref{the_model} with some $M^*$. Here, the behavior of the distance strongly depends on $M^*$. We consider the following two cases: (i) $M^*$ is finite and (ii) $M^*$ is infinite. In case (i), the consistency of $\hat f_n$ to $f^*$ is shown with an explicit convergence rate (Theorem \ref{thm:converge}). More surprisingly, the consistency also holds in case (ii) with a mild assumption (Theorem \ref{thm:high_model}).

Figure~\ref{space} illustrates the relations of these functions and the functional space.
The rectangular areas are the classes of functions represented by AMNR
with small $M^*$, AMNR with large $M^*$, and Sobolev space
$\mathcal{W}^{\beta}$ with rank additivity.

\begin{figure}[tb]
  \begin{center}
  \fbox{
  \includegraphics[width=60mm]{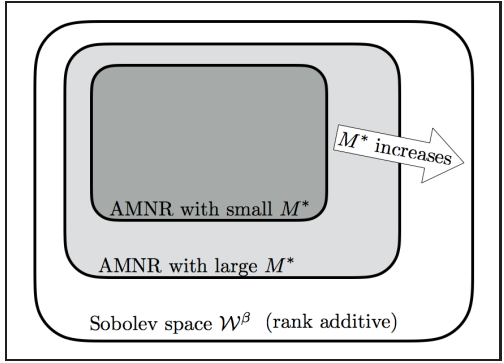}
  }
  \end{center}
  \caption{Functional space and the effect of $M^*$.}
  \label{space}
\end{figure}

Note that the formal assumptions and proofs of this section
are shown in supplementary material.

\subsection{Estimation with Finite $M^*$} \label{theory_est}

The consistency of Bayesian nonparametric estimators is evaluated in terms of posterior consistency
\cite{GGvdV2000,GvdV2007,vdVvZ2008}.  Here, we follow the same strategy. Let $\|f\|_n^2 :=
\frac{1}{n}\sum_{i=1}^n f(x_i)^2$ be the empirical norm. We
define $\epsilon_n^{(k)}$ as the contraction rate of the estimator of
local function $f^{(k)}$, which evaluates the probability mass of the
GP around the true function.  Note that the order of
$\epsilon_n^{(k)}$ depends on the covariance kernel function of the GP
prior, in which the optimal rate of $\epsilon_n^{(k)}$ is
given by \eqref{general-rate} with $d=I_k$.
For brevity, we suppose that the variance of the noise $u_i$ is known and the kernel in the GP prior is selected to be optimal.\footnote{We assume that the Mat\'{e}rn kernel is selected and the weight of the kernel is equal to $\beta$. Under these conditions, the optimal rate is achieved \citep{tsybakov2008}.}
Then, we obtain the following result.

\begin{them}[Convergence analysis] \label{thm:converge} Let $M=M^* < \infty$.
  Then, with Assumption \ref{asmp:regular_volume} and some finite constant $C>0$,
	\begin{align*}
		E \|\hat{f}_n - f^* \|_n^2 \leq C n^{-2\beta/(2\beta + \max_k I_k)}.
	\end{align*}	
\end{them}

Theorem \ref{thm:converge} claims the validity of the estimator \eqref{GP_est}.
Its convergence rate corresponds to
the minimax optimal rate of estimating a function in $\mathcal{W}^{\beta}$ on
compact support in $\mathbb{R}^{I_k}$, showing that the convergence rate of 
AMNR depends only on the largest dimensionality of $X$.



\subsection{Estimation with Infinite $M^*$} \label{theory_est_without}

When $M^*$ is infinite, we cannot use the same strategy used in Section \ref{theory_est}. Instead, we truncate $M^*$ by finite $M$ and evaluate the bias error caused by the truncation. To evaluate the bias, we assume that the local functions are in descending order of their volumes $\bar{f}_m := \sum_r \prod_k f_m^{(k)}$, i.e., $\{\bar f_m\}$ are ordered as satisfying $\|\bar{f}_{m'}\|_2 \geq \|\bar{f}_{m}\|_2$ for all $m' > m$.
We then introduce the assumption that $\|\bar{f}_m\|$ decays to zero polynomially with respect to $m$.
\begin{asmp}\label{asmp:regular_volume}
	With some constant $\gamma > 0$,
	\begin{align*}
		\|\bar{f}_m\|_2 = O\left( m^{-\gamma} \right),
	\end{align*}
	as $m \rightarrow \infty$.
\end{asmp}
This assumption is sufficiently limited. For example, if we pick $\{f^{(k)}_m\}$ as $\sum_r \prod_{k=1}^K f_{m}^{(k)}(x_r^{(k)})$ are orthogonal to each $m$,\footnote{For example, we can obtain such $\{f^{(k)}_m\}$ by the Gram–Schmidt process.} then
the Parseval-type inequality leads to $\sum_{m}\|\bar{f}_m\|_2 = \|\sum_m \bar{f}_m\|_2 = \|f^*\|_2 < \infty$, implying $\|\bar f_m\|_2 \to 0$ as $m\to \infty$.

Then we claim the main result in this section. 
\begin{them} \label{thm:high_model}
	Suppose we construct the estimator \eqref{GP_est} with
	\begin{align*}
		M \asymp (n^{-2 \beta/(2\beta + \max_k I_k)})^{\gamma/(1+\gamma)},
	\end{align*}
where $\asymp$ denotes equality up to a constant. Then, with some finite constant $C>0$, 
	\begin{align*}
		E \|\hat{f}_n - f^* \|_n^2 \leq  C (n^{-2 \beta/(2\beta + \max_k I_k) })^{\gamma/(1+\gamma)}.
	\end{align*}	
\end{them}
The above theorem states that, even if we truncate $M^*$ by finite $M$, the convergence rate is nearly the same as the case of finite $M^*$ (Theorem \ref{thm:converge}), which is slightly worsened by the factor $\gamma / (1+\gamma)$.

Theorem \ref{thm:high_model} also suggests how to determine $M$.
For example, if $\gamma = 2, \beta = 1$, and $\max_k I_k = 100$,
$M \asymp  n^{1/70} $ is recommended, which is much smaller than the sample size.
Our experimental results (Section \ref{sec:experiments_rate}) also support this. Practically, very small $M$ is sufficient, such as $1$ or $2$, even if $n$ is greater than $300$.

Here, we show the conditional consistency of AMNR, which is directly derived from Theorem \ref{thm:high_model}.
\begin{cor}
	For all function $f^* \in \mathcal{W}^{\beta}$ with finite $M=M^*$ or Assumption \ref{asmp:regular_volume}, the estimator \eqref{GP_est} satisfies
	\begin{align*}
			E \|\hat{f}_n - f^0 \|_n^2 \rightarrow 0,
	\end{align*}
	as $n \rightarrow \infty$.
\end{cor}

\section{Related Work}

\subsection{Nonparametric Tensor Regression}
The \emph{tensor GP (TGP)} \citep{zhao2014,hou2015a} in a method that estimates the function in $\mathcal{W}^{\beta}(\mathcal{S})$ directly.
TGP is essentially a GP regression model that flattens a tensor into a high-dimensional vector and takes the vector as an input. \citet{zhao2014} proposed its estimator and applied it to image recognition from a monitoring camera. \citet{hou2015a} applied the method to analyze brain signals.
Although both studies demonstrated the high performance of TGP, its theoretical aspects such as convergence have not been discussed.

\citet{Signoretto2013} proposed a regression model
with tensor product reproducing kernel Hilbert spaces
(TP-RKHSs). Given a set of vectors $\{x_k\}_{k=1}^K$, their model is
written as
\begin{align}
	\sum_{j} \alpha_{j} \prod_k f_j^{(k)}(x_k), \label{model:signoretto}
\end{align}
where $\alpha_j$ is a weight. 
The key difference between TP-RKHSs (\ref{model:signoretto}) and AMNR
is in the input. TP-RKHSs take only a single vector for each order,
meaning that the input is implicitly assumed as rank one. On the other
hand, AMNR takes rank-$R$ tensors where $R$ can be greater than
one. This difference allows AMNR to be used for more general purposes,
because the tensor rank observed in the real world is mostly greater
than one. Furthermore, the properties of the estimator, such as
convergence, have not been investigated.

\subsection{TLR}
For the matrix case ($K=2$), 
\citet{dyrholm2007} proposed a classification model as
\eqref{lin_model}, where $B$ is assumed to be low rank. 
\citet{hung2013} proposed a logistic regression where the
expectation is given by \eqref{lin_model} and $B$ is a rank-one
matrix. \citet{zhou2013} extended these concepts for tensor
inputs. \citet{suzuki2015} and \citet{guhaniyogi2015} proposed a Bayes
estimator of TLR and investigated its convergence rate.

Interestingly, AMNR is interpretable as a piecewise nonparametrization
of TLR.  Suppose $B$ and $X$ have rank-$M$ and rank-$R$ CP
decompositions, respectively.  The inner product \emph{in the tensor
  space} in (\ref{lin_model}) is then rewritten as the product of the
inner product \emph{in the low-dimensional vector space}, i.e.,
\begin{align}
  \label{eq:low-rank linear}
	\langle B,X_i \rangle = \sum_{m=1}^M \sum_{r=1}^R \lambda_{r,i} \prod_{k=1}^K \langle b_m^{(k)},x_{r,i}^{(k)} \rangle,
\end{align}
where $b_m^{(k)}$ is the order-$K$ decomposed vector of
$B$. The AMNR formation is obtained by replacing the inner product $\langle b_m^{(k)},x_r^{(k)}
\rangle$ with local function $f_m^{(k)}$.

From this perspective, we see that AMNR incorporates the advantages of
TLR and TGP.  AMNR captures nonlinear relations between $Y$ and $X$
through $f_m^{(k)}$, which is impossible for TLR due to its
linearity. Nevertheless, in contrast to TGP, an input of the function
constructed in a nonparametric way is given by an $I_k$-dimensional
vector rather than an $(I_1,\dots,I_K)$-dimensional tensor. This
reduces the dimension of the function's support and significantly
improves the convergence rate (Section~\ref{sec:theory}).

\subsection{Other studies}
\citet{koltchinskii2010} and \citet{suzuki2013} investigated
\emph{Multiple Kernel Learning (MKL)} considering a nonparametric
$p$-variate regression model with an additive structure:
$\sum_{j=1}^p f_j(x_j)$.
To handle high dimensional inputs, MKL reduces the input dimensionality by the additive structure for $f_j$ and $x_j$.
Note that both studies deal with a vector input, and they do not fit to tensor regression analysis.

\begin{table}[tbph]
\caption{Comparison of related methods. \label{fig:compare}}
\label{sample-table}
\vskip 0.15in
\begin{center}
\begin{small}
\begin{sc}
\begin{tabular}{lcccr}
\hline
\abovespace\belowspace
Method & \shortstack{Tensor\\Input}  & \shortstack{Non-\\linearity} & \shortstack{Convergence\\Rate with \eqref{general-rate}}  \\
\hline
\abovespace
TLR   		 &	$\surd$ 		& 					& $d = 0$ \\
TGP 			&	$\surd$ 		& $\surd$ 	&	$d = \prod_k I_k$			 \\
TP-RKHSs & 	rank-$1$		& $\surd$	& N/A \\
MKL		    &						& $\surd$	& N/A \\
AMNR    	&	$\surd$ 		& $\surd$	& $d = \max I_k$ \\
\hline
\end{tabular}
\end{sc}
\end{small}
\end{center}
\vskip -0.1in
\end{table}

Table \ref{fig:compare} summarizes the methods introduced in this section. As shown,
MKL and TP-RKHSs are not applicable for general tensor input.
In contrast, TLR, TGP, and AMNR can take multi-rank tensor data as inputs, and their applicability is much wider. Among the three methods, AMNR is only the one that manages nonlinearity and avoids the curse of dimensionality on tensors.


\begin{figure*}[tb]
\begin{center}
  \includegraphics[width = 50mm]{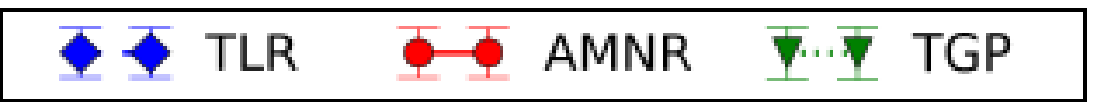}
\end{center}
\begin{minipage}{0.49\hsize}
	\begin{minipage}{0.49\hsize}
	  \includegraphics[width=42mm]{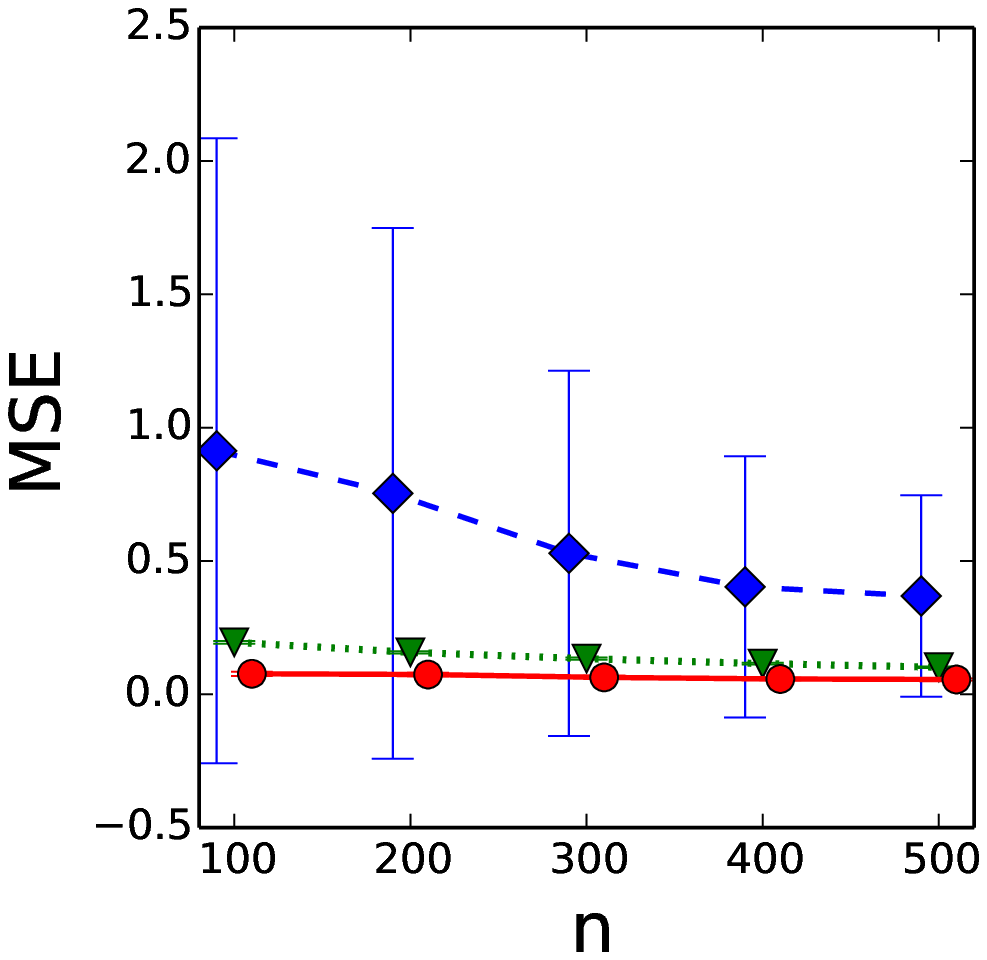}
	  \begin{center}
	  	(a) Full view
	  \end{center}
	\end{minipage}
	\begin{minipage}{0.49\hsize}
	  \includegraphics[width=42mm]{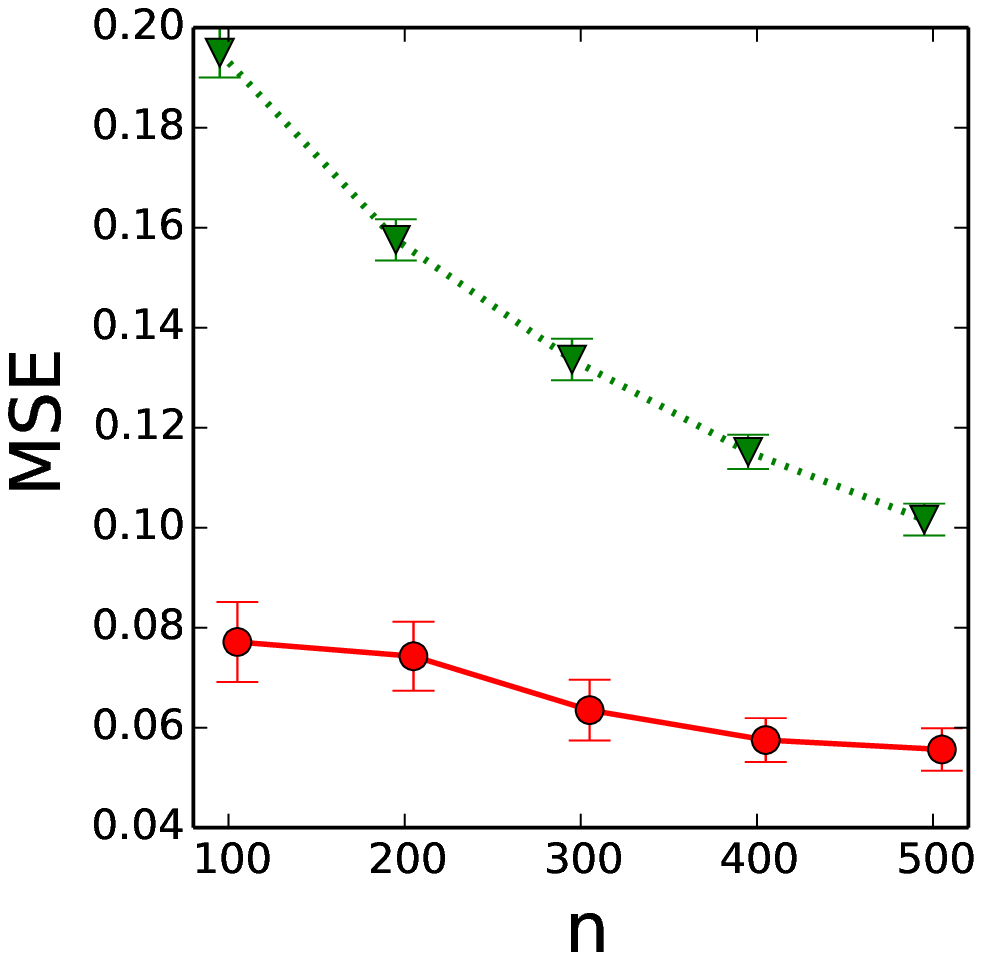}
	  \begin{center}
	  	(b) Enlarged view
	  \end{center}
	\end{minipage}
	 \caption{Synthetic data experiment: Low-rank data.\label{exp1}}
\end{minipage}
\begin{minipage}{0.49\hsize}
	\begin{minipage}{0.49\hsize}
	  \includegraphics[width=42mm]{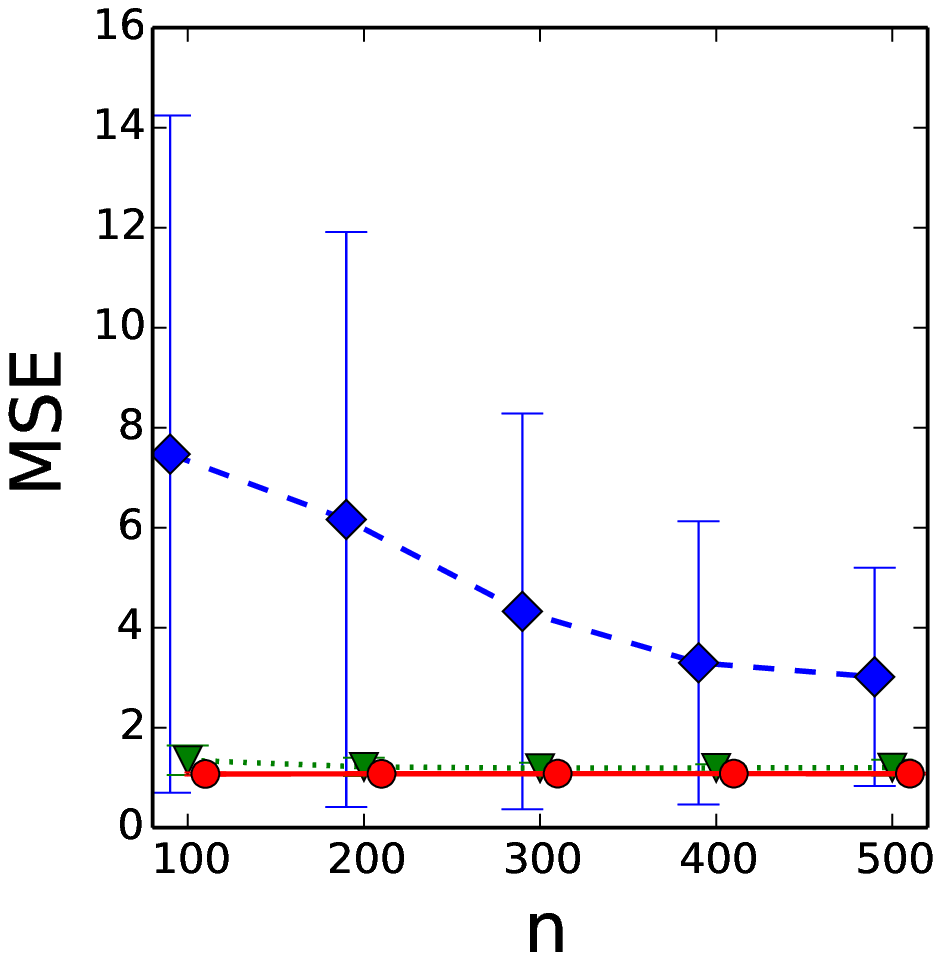}
	  \begin{center}
	  	(a) Full view
	  \end{center}
	\end{minipage}
	\begin{minipage}{0.49\hsize}
	  \includegraphics[width=42mm]{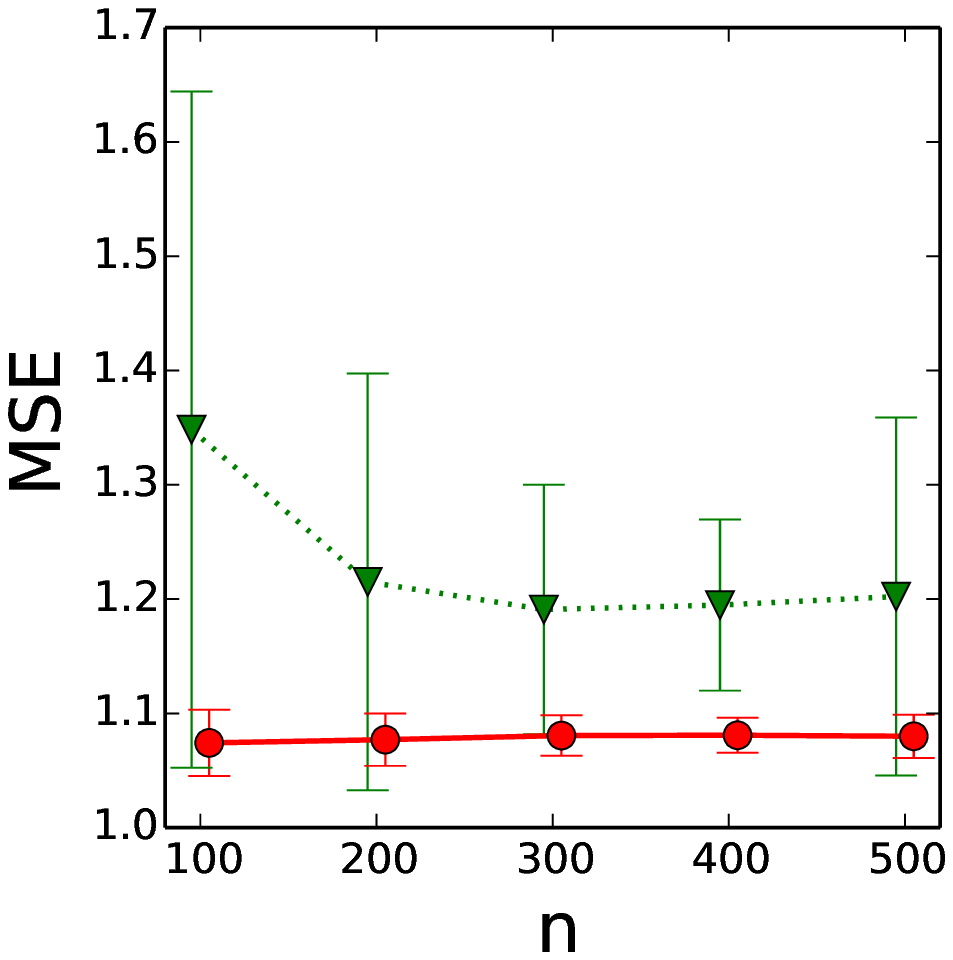}
	  \begin{center}
	  	(b) Enlarged view
	  \end{center}
	\end{minipage}
	  \caption{Synthetic data experiment: Full-rank data.\label{exp3}}
\end{minipage}
\end{figure*}

\begin{figure*}[tb]
\begin{minipage}{0.65\hsize}
	\begin{center}
		  \includegraphics[width = 50mm]{./pic/legend.eps}	
	\end{center}
	\begin{minipage}{0.99\hsize}
		\begin{minipage}{0.49\hsize}
		  \begin{center}
		  \includegraphics[width=42mm]{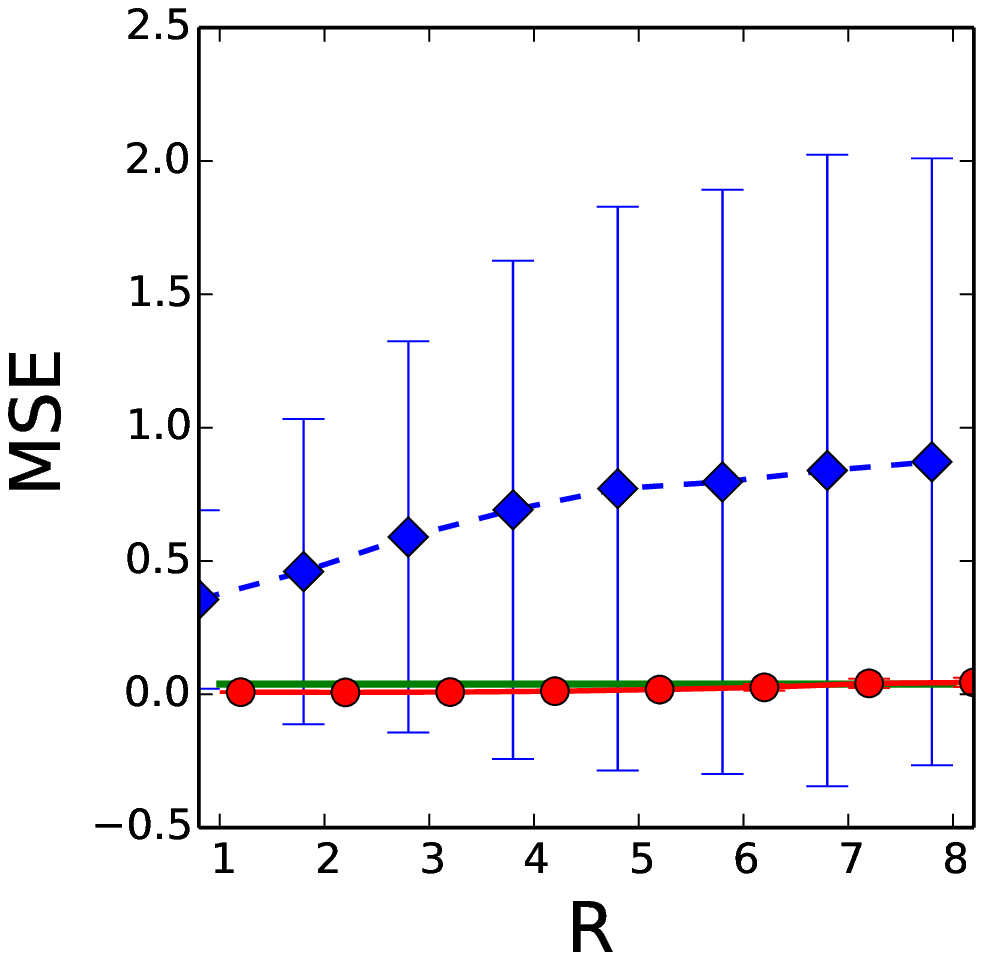}\\
		  	(a) Full view
		  \end{center}
		\end{minipage}
		\begin{minipage}{0.49\hsize}
		  \begin{center}
		  \includegraphics[width=42mm]{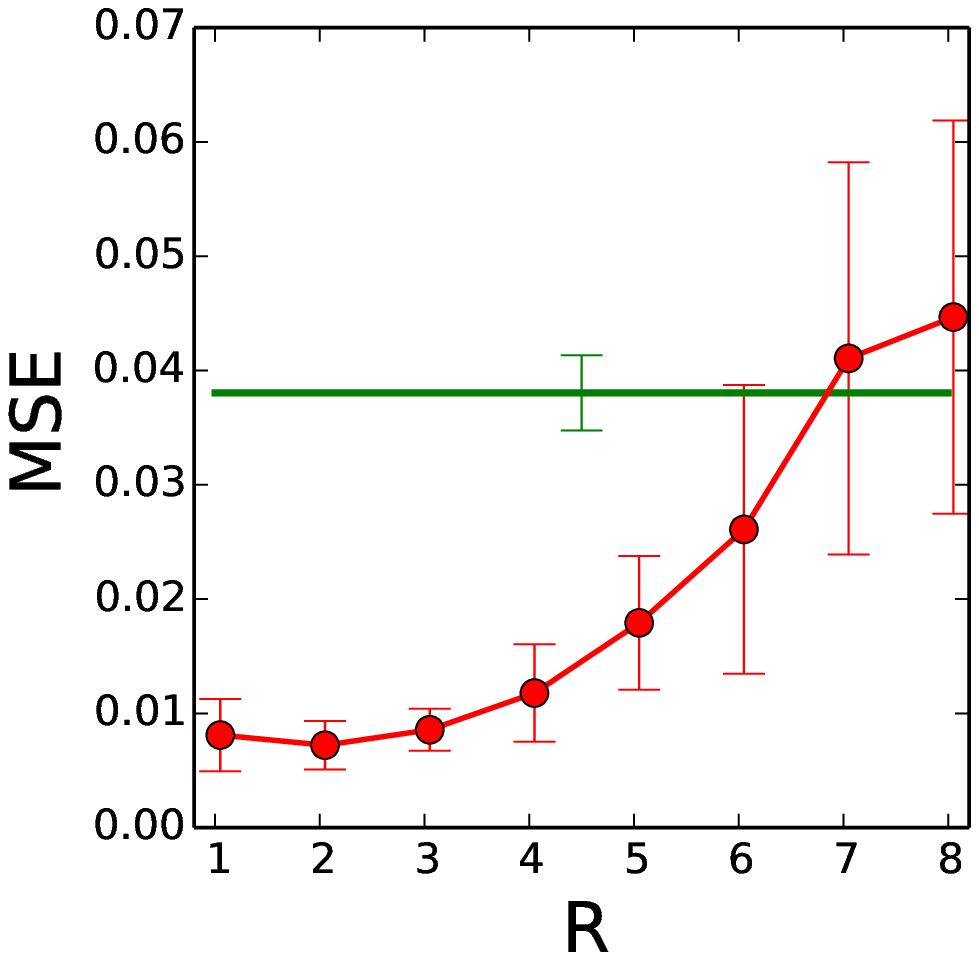}\\
		  	(b) Enlarged view
		  \end{center}
		\end{minipage}
    	\caption{Synthetic data experiment: Sensitivity of $R$.\label{expR}}
	\end{minipage}
\end{minipage}
\begin{minipage}{0.35\hsize}
	  \begin{center}
	 \includegraphics[width=42mm]{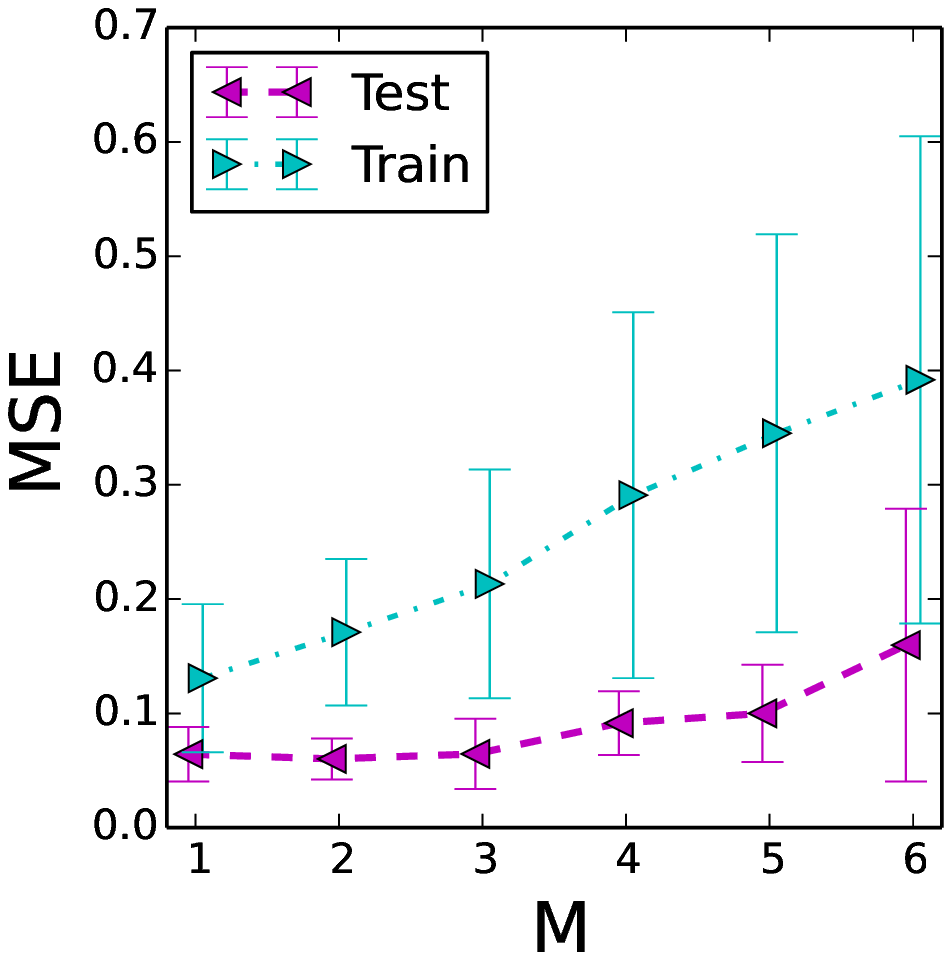}\\
	  \caption{Synthetic data experiment: Sensitivity of $M$.\label{expM}}
	  \end{center}
\end{minipage}
\end{figure*}

\section{Experiments} \label{sec:experiments}

\subsection{Synthetic Data} \label{nume_exp} 

We compare the prediction performance of three models: TLR, TGP, and
AMNR. In all experiments, we generate datasets by the data generating
process (dgp) as $Y = f^*(X) + u$ and fix the noise variance as
$\sigma^2 = 1$. We set the size of $X \in \mathbb{R}^{20 \times 20}$,
i.e., $K = 2$ and $I_1=I_2 = 20$.
By varying the sample size as $n \in \{100,200,300,400,500\}$, we evaluate
the empirical risks by the mean-squared-error (MSE) for the testing
data, for which we use one-half of the samples.
For each experiment, we derive the mean and variance of the MSEs in
100 trials.
For TGP and AMNR, we optimize the bandwidth of the kernel function by
grid search in the training phase.

\subsubsection{Low-Rank Data}\label{sec:low-rank-model}

First, we consider the case that $X$ and the dgp are exactly
low rank. We set $R^*$, the true rank of $X$, as $R^* = 4$ and
$$f^*(X) = \sum_{r=1}^{R^*} \lambda_r
\prod_{k=1}^K (1 + \exp(–\gamma^T x_r^{(k)}))^{-1}$$
where $[\gamma]_j=0.1 j$.  The results (Figure~\ref{exp1}) show that
AMNR and TGP clearly outperform TLR, implying that they
successfully capture the nonlinearity of the true function.
%
%
To closely examine the difference between AMNR and TGP, we enlarge the
corresponding part (Figure~\ref{exp1}(b)),
which shows that AMNR consistently outperforms TGP. Note that the
performance of TGP improves gradually as $n$ increases, implying that
the sample size is insufficient for TGP due to its slow
convergence rate.

\subsubsection{Full-Rank Data}\label{sec:full-rank-model}

Next, we consider the case that $X$ is full rank and the dgp
has no low-rank structure, i.e., model misspecification will occur
in TLR and AMNR.  We generate $X$ as $X_{j_1 j_2} \sim
\mathcal{N}(0,1)$ with 
$$f^*(X) = \prod_{k=1}^K
(1 + \exp( - \|X\|_2/\prod_k I_k))^{-1}.$$ 
The results
(Figure~\ref{exp3}) show that, as in the previous
experiment, AMNR and TGP outperform TLR. Although the difference
between AMNR and TGP (Figure \ref{exp3}(b)) is much smaller,
AMNR still outperforms TGP. This implies that, while the
effect of AMNR's model misspecification is not negligible, TGP's slow
convergence rate is more problematic.

\subsection{Sensitivity of Hyperparameters}

Here, we investigate how the truncation of $R^*$
and $M^*$ affect prediction performance. In the following
experiments, we fix the sample size as $n = 300$.

First, we investigate the sensitivity of $R$. We use the same low-rank dgp used
in Section~\ref{sec:low-rank-model} (i.e., $R^*=4$.)
The results (Figure~\ref{expR}) show that AMNR and TGP
clearly outperform TLR.
%
Although their performance is close, AMNR beats TGP when $R$ is not
too large, implying that the negative effect of truncating $R^*$ is
limited.

Next, we investigate the sensitivity of $M$.  We use the same
full-rank dgp used in
Section~\ref{sec:full-rank-model}. Figure~\ref{expM} compares the
training and testing MSEs of AMNR, showing that both errors increase
as $M$ increases. These results imply that the model bias decreases
quickly and estimation error is more dominant. Indeed, the lowest
testing MSE is achieved at $M=2$. This agrees satisfactory with the
analysis in Section \ref{theory_est_without}, which recommends small
$M$.


\subsection{Convergence Rate}\label{sec:experiments_rate}
Here, we confirm how the empirical convergence rates of AMNR and TGP
meet the theoretical convergence rates. To clarify the relation, we
generate synthetic data from dgp with $\beta = 1$ such that the
difference between TGP and AMNR is maximized. To do so, we design the dgp
function as 
$f^*(X) = \sum_{r=1}^R \prod_{k=1}^K f^{(k)} (x_r^{(k)})$
and 
$$f^{(k)} = \sum_l \mu_l \phi_l(\gamma^T x),$$
where $\phi_l(z) =
\sqrt{2} \cos ((l-1/2)\pi z)$ is an orthonormal basis function of the
functional space and $\mu_l = l^{-3/2}\sin (l)$.\footnote{This dgp is
  derived from a theory of Sobolev ellipsoid; see
  \cite{tsybakov2008}.}

\begin{table}[h]
\vskip 0.15in
\begin{center}
\begin{small}
\begin{sc}
\begin{tabular}{lccccccc}
\hline
Setting & $K$ & $R^*$ & $I_1$ & $I_2$& $I_3$ & \multicolumn{2}{c}{$d$ in \eqref{general-rate}}\\
No. &  &  &  & &  & TGP & AMNR\\
\hline
\abovespace
(i)   		& $3$ & $2$ & $10$&$10$&$10$ & $1000$ & $10$\\
(ii)		& $3$ & $2$ & $3$&$3$&$3$ & $27$ & $3$\\
(iii)   	& $3$ & $2$ &$10$&$3$&$3$ & $90$ & $10$ \\
\hline
\end{tabular}
\end{sc}
\end{small}
\end{center}
\caption{Synthetic data experiment: Settings for convergence rate.\label{table:setting}}
\vskip -0.1in
\end{table}

For $X$ we consider three variations: $3\times 3\times 3$, $10\times,
3\times 3$, and $10\times 10\times 10$
(Table~\ref{table:setting}). Figure~\ref{conv_rate} shows testing MSEs
averaged over $100$ trials. The theoretical convergence rates are also
depicted by the dashed line (TGP) and the solid line (AMNR). To align
the theoretical and empirical rates, we adjust them at $n=50$. The
result demonstrates the theoretical rates agree with the practical
performance.

\begin{figure}[tbhp]
	\begin{minipage}{0.24\hsize}
		\begin{center}
	  		\includegraphics[width=0.99\hsize]{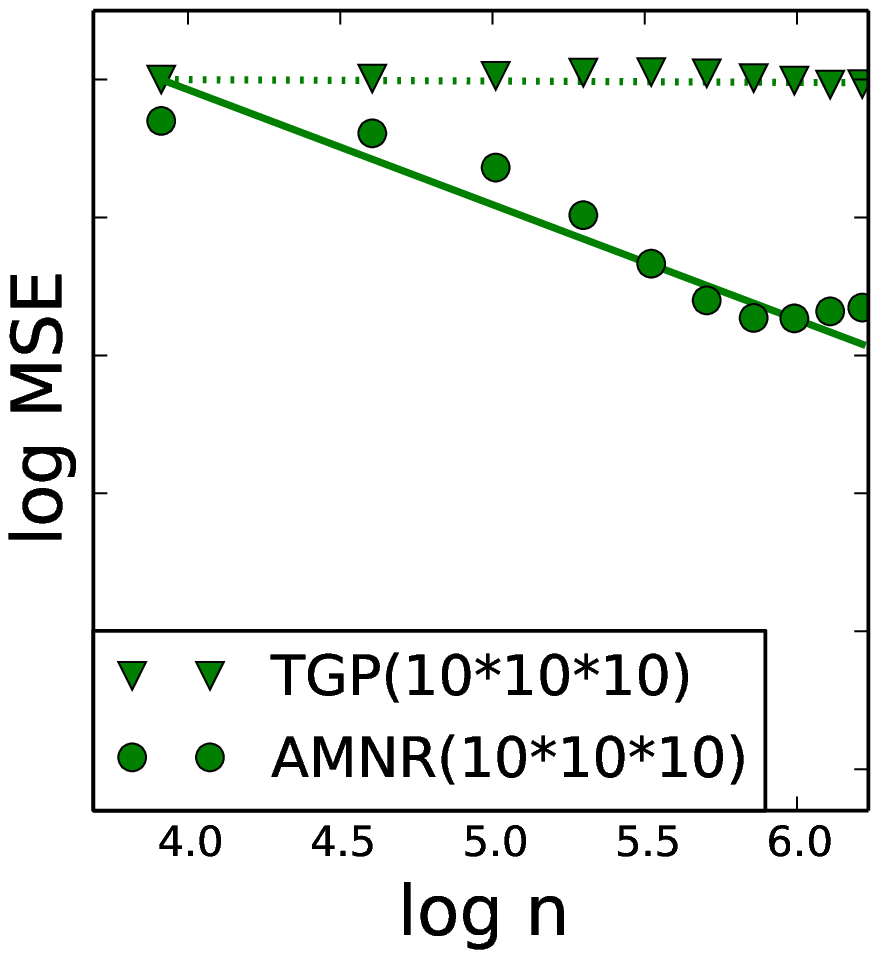}\\
	  		(I) $10 \times 10 \times 10$ tensor.
  		\end{center}
	\end{minipage}
	\begin{minipage}{0.24\hsize}
		\begin{center}
	  		\includegraphics[width=0.99\hsize]{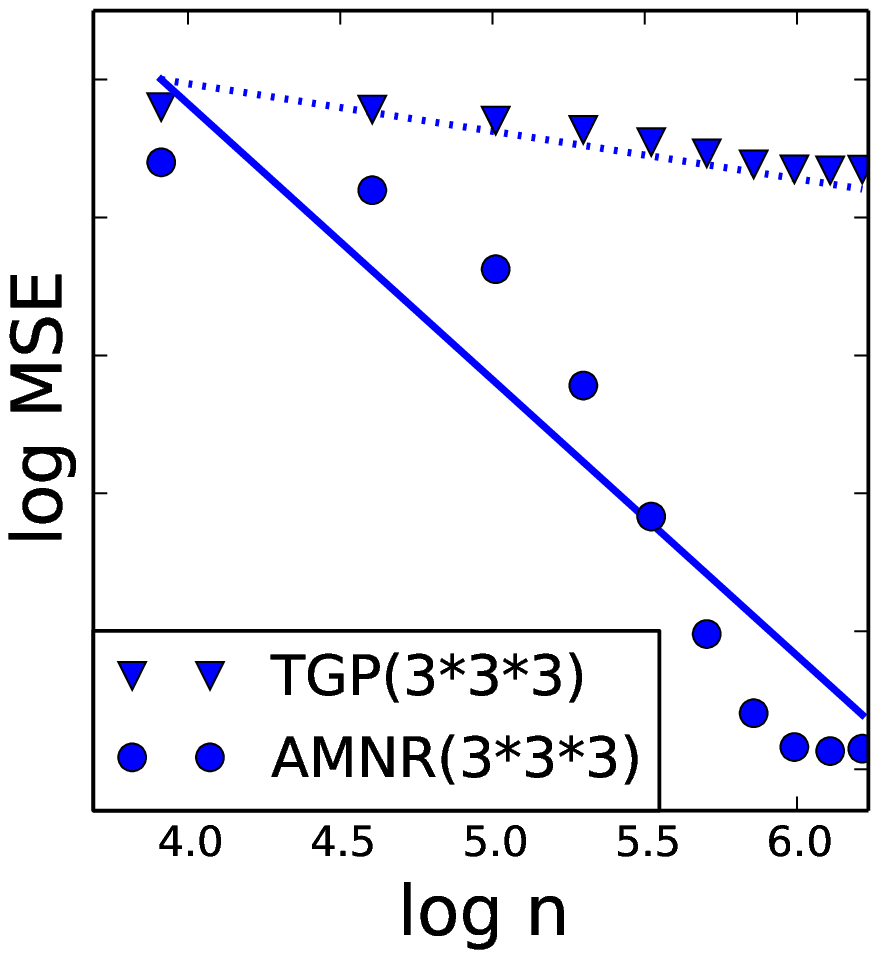}\\
	  		(II) $3 \times 3 \times 3$\\ tensor.
  		\end{center}
	\end{minipage}
	\begin{minipage}{0.24\hsize}
		\begin{center}
	  		\includegraphics[width=0.99\hsize]{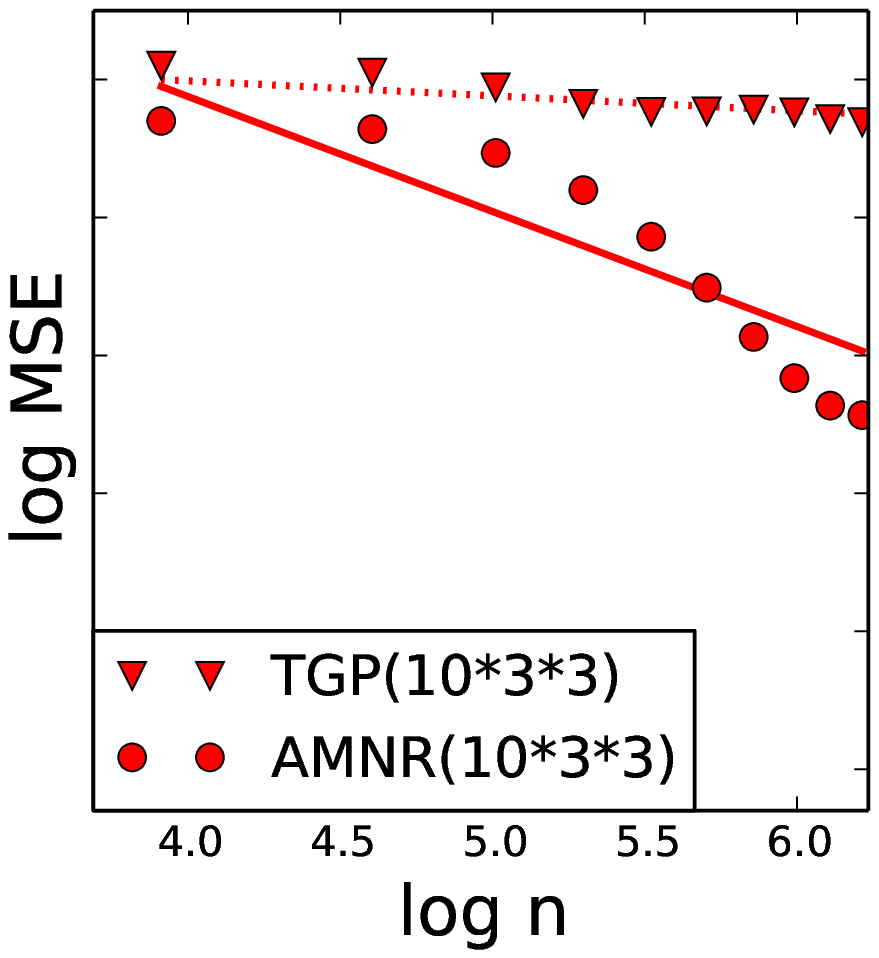}\\
	  		(III) $10 \times 3 \times 3$ tensor.
  		\end{center}
	\end{minipage}
	\begin{minipage}{0.24\hsize}
		\begin{center}
	  		\fbox{\includegraphics[width=0.69\hsize]{./pic/legend.eps}}
  		\end{center}		
	\end{minipage}
\caption{Comparison of convergence rate.\label{conv_rate}}
\end{figure}

\subsection{Prediction of Epidemic Spreading}

Here, we deal with the \emph{epidemic spreading} problem in complex
networks~\cite{anderson1992,alessandro2012} as a matrix regression
problem. More precisely, given an adjacency matrix network $X_i$, we simulate the spreading
process of a disease by the susceptible-infected-recovered (SIR) model
as follows.
\begin{enumerate}
\item We select 10 nodes as the initially infected nodes.
\item The nodes adjacent to the infected nodes become infected with
  probability $0.01$.
\item Repeat Step 2. After 10 epochs, the infected nodes recover and are no longer infected (one iteration = one epoch).
\end{enumerate}
After the convergence of the above process, we count the total number of
infected nodes as $Y_i$.
Note that the number of infected nodes depends strongly on the network
structure and its prediction is not trivial. Conducting the simulation
is of course a reliable approach; however, it is time-consuming, especially
for large-scale networks. In contrast, once we obtain a trained model,
regression methods make prediction very quick.

As a real network, we use the Enron email dataset~\cite{klimt2004}, which is a collection of emails. We consider these emails as
undirected links between senders and recipients (i.e., this is a
problem of estimating the number of email addresses infected by an
email virus).  First, to reduce the network size, we select the top
$1,000$ email addresses based on frequency and delete emails sent to
and received from other addresses. After sorting the remaining emails
by timestamp, we sequentially construct an adjacency matrix from
every $2,000$ emails, and we finally obtain $220$ input matrices.

For the analysis, we set $R=2$ for the AMNR and TLR
methods.\footnote{We also tested $R = 1,2,4$, and $8$; however, the
  results were nearly the same.} Although $R=2$ seems small, we can
still use the top-two eigenvalues and eigenvectors, which contain a
large amount of information about the original tensor. In addition,
the top eigenvectors are closely related to the threshold of outbreaks
in infection networks~\cite{wang2003}. From these perspectives, the
good performance demonstrated by AMNR with $R=2$ is reasonable.
The bandwidth of the kernel is optimized by grid search in the
training phase.  

Figure~\ref{exp_net} shows the training and testing MSEs. Firstly,
there is a huge performance gap between TLR and the nonparametric
models in the testing error. This indicates that the relation between
epidemic spreading and a graph structure is nonlinear and the linear
model is deficient for this problem. Secondly, AMNR outperforms TGP
for every $n$ in both training and testing errors. In addition, the
performance of AMNR is constantly good and almost unaffected by $n$.
This suggests that the problem has some extrinsic information that
inflates the dimensionality, and the efficiency of TGP is
diminished. On the other hand, it seems AMNR successfully captures the
intrinsic information in a low-dimensional space by its ``double
decomposition'' so that AMNR achieves the low-bias and low-variance
estimation.

\begin{figure}[tbhp]
\begin{center}
  \includegraphics[width = 50mm]{./pic/legend.eps}
\end{center}
\begin{minipage}{0.48\hsize}
  \begin{center}
  \includegraphics[width=41mm]{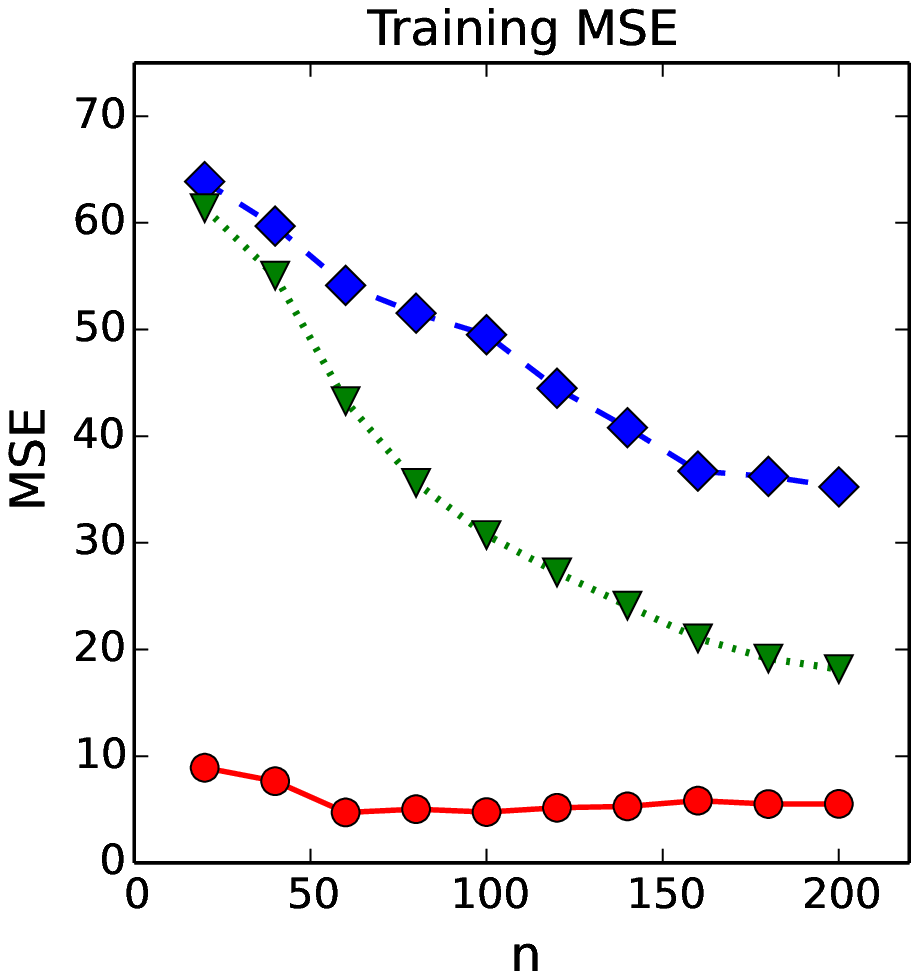}
  \end{center}
\end{minipage}
\begin{minipage}{0.48\hsize}
  \begin{center}
  \includegraphics[width=41mm]{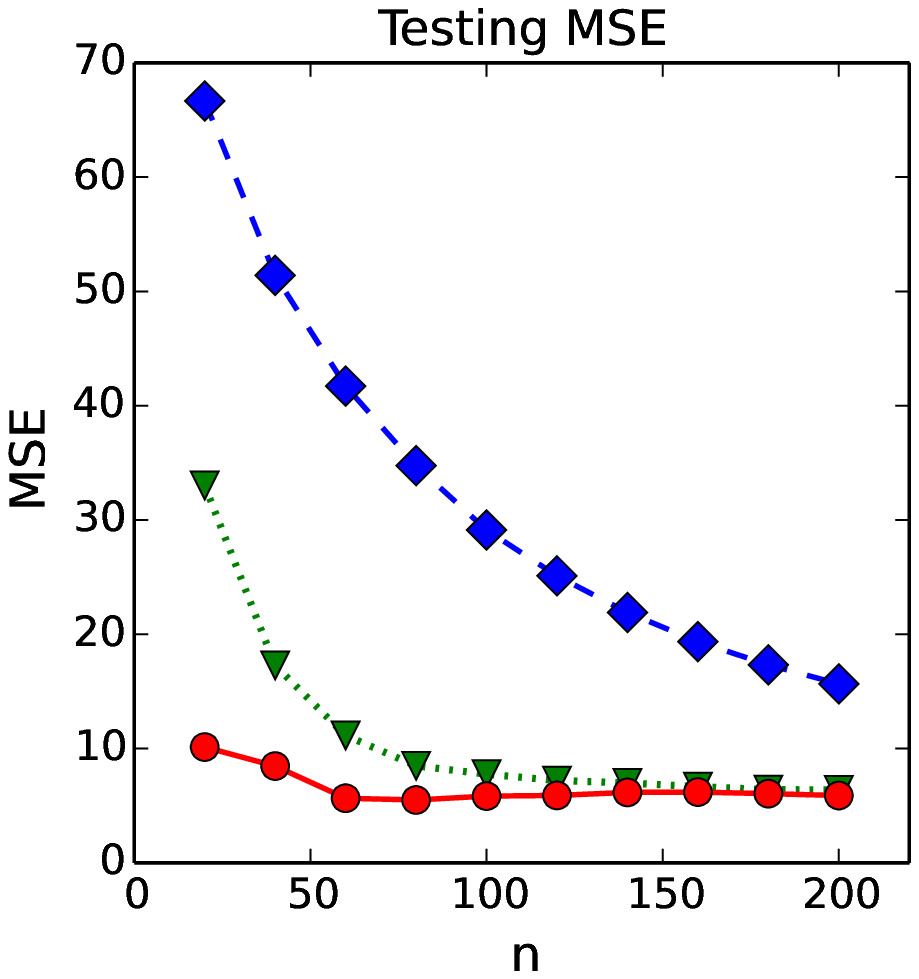}
  \end{center}
\end{minipage}
\caption{Epidemic spreading experiment: Prediction performance.}
  \label{exp_net}
\end{figure}


\section{Conclusion and Discussion}

We have proposed AMNR, a new nonparametric model for the tensor regression problem.
We constructed the regression function as the sum-product of the local functions,
and developed an estimation method of AMNR. 
We have verified our theoretical analysis and demonstrated that AMNR was the most accurate method for predicting the spread of an epidemic in a real network.

The most important limitation of AMNR is the computational complexity, which is better than TGP but worse than TLR. The time complexity of AMNR with the GP estimator
is $O(nR\prod_k I_k + M (nR)^3 \sum_k I_k )$, where CP decomposition
requires $O(R\prod_k I_k)$~\cite{Anh2013} for $n$ inputs and the GP
prior requires $O(I_k (nR)^3)$ for the $MK$ local functions. On the
contrary, TGP requires $O(n^3 \prod_k I_k)$
computation because it must evaluate all the elements of $X$ to
construct the kernel Gram matrix.
When $R \ll n^2$ and $MR^3\sum_k I_k \ll \prod_k I_k$, which are
satisfied in many practical situations, the proposed method is more efficient
than TGP.

Approximation methods
for GP regression can be used to reduce the computational burden of AMNR. For example, \citet{williams2001using} proposed the Nystr{\"o}m method,
which approximates the kernel Gram matrix by a low-rank matrix. If we
apply rank-$L$ approximation, the computational cost of AMNR can be
reduced to $O((L^3 + nL^2)\prod_k I_k)$.


\newpage
\appendix

\section{Proof of Theorem~\ref{thm:converge}}
Here, we describe the detail and proof of Theorem 1. At the beginning, we
introduce a general theory for evaluating the convergence of a
Bayesian estimator.

Preliminary, we introduce some theorems from previous studies. Let
$P_0$ be a true distribution of $X$, $K(f,g)$ be the Kullback-Leibler
divergence, and define $V(f,g) = \int (\log (f/g))^2 f dx$. Let $d$ be
the Hellinger distance, $N(\epsilon,\mathcal{P},d)$ be the bracketing
number, and $D(\epsilon,\mathcal{P},d)$ be the packing number. Also we
consider a reproducing kernel Hilbert space (RKHS), which is a closure
of linear space spanned by a kernel function.
Denote by $\mathcal{H}^{(k)}$ the RKHS on
$\mathcal{X}^{(k)}$.

The following theorem provides a novel tool to evaluate the Bayesian
estimator by posterior contraction.

\begin{them}[Theorem 2.1 in \cite{GGvdV2000}] 
  Consider a posterior distribution $\Pi_n(\cdot | D_n)$ on a set
  $\mathcal{P}$.  Let $\epsilon_n$ be a sequence such that $\epsilon_n
  \rightarrow 0$ and $n \epsilon_n^2 \rightarrow \infty$.  Suppose
  that, for a constant $C > 0$ and sets $\mathcal{P}_n \subset
  \mathcal{P}$, we have
	\begin{enumerate}
		\item $\log D(\epsilon_n,\mathcal{P}_n,d) \leq n \epsilon_n^2$,
		\item $\Pi_n(\mathcal{P}_n \backslash \mathcal{P}) \leq \exp(-n \epsilon_n^2 (C + 4))$,
		\item $\Pi_n \left( p : -K(p,p_0) \leq \epsilon_n^2, V(p,p_0) \leq \epsilon_n^2 \right) \geq \exp(-Cn\epsilon_n^2)$.
	\end{enumerate}
	Then, for sufficiently large $C'$, $E \Pi_n(P:d(P,P_0) \geq C' \epsilon_n)|D_n) \rightarrow 0$.	
\end{them}

Based on the theorem, \cite{vdVvZ2008} provide a more useful result
for the Bayesian estimator with the GP prior.  They consider the
estimator for an infinite dimensional parameter with the GP prior and
investigated the posterior contraction of the estimator with GP.  They
provide the following conditions.
\paragraph{Condition (A)}
With some Banach space
$(\mathcal{B},\|\cdot\|)$ and RKHS $(\mathcal{H},\|\cdot\|)$, 
\begin{enumerate}
	\item $\log N(\epsilon_n,B_n,\|\cdot\|) \leq C n \epsilon_n^2$,
	\item $\mbox{Pr}(W \notin B_n) \leq \exp(-Cn \epsilon_n^2 )$,
	\item $\mbox{Pr}\left(\|W - w_0\| < 2 \epsilon_n \right) \geq \exp(-Cn\epsilon_n^2)$,
\end{enumerate}
where $W$ is a random element in $\mathcal{B}$ and $w_0$ is a true
function in support of $W$. 

When the estimator with the GP prior
satisfied the above conditions, the posterior contraction of the
estimator is obtained as Theorem 2.1 in \cite{GGvdV2000}.

Based on this, we obtain the following result.  Consider a set of GP
$\{ \{W_{m,x}^{(k)}:x \in \mathcal{X}^{(k)}\}\}_{m=1,\ldots,M,k =
  1,\ldots,K}$ and let $\{F_x:x \in \mathcal{X}\}$ be a stochastic
process that satisfies $F_x = \sum_m \sum_r \lambda_r \prod_k
W_{m,x_r^{(k)}}^{(k)}$.

Also we assume
that there exists a true function $f_0$ which is constituted by a
unique set of local functions $\{w_{m,0}^{(k)}\}_{m=1,\ldots,M,k =
  1,\ldots,K}$.

To describe posterior contraction, we define a contraction rate
$\epsilon_n^{(k)}$. It converges to zero as $n \rightarrow
\infty$. Let $\phi^{(k)}(\epsilon)$ be a concentration function such that
\begin{align*}
	\phi^{(k)}(\epsilon) := \inf_{h \in \mathcal{H}^{(k)}:\|h-w_0\|<\epsilon}\|h\|_{\mathcal{H}^{(k)}}^2 - \log \mbox{Pr}(\|W^{(k)}\|<\epsilon),
\end{align*}
where $\|\cdot\|_{\mathcal{H}^{(k)}}$ is the norm induced by the inner
product of RKHS. 
We define the contraction rate with $\phi^{(k)}(\epsilon)$.
We denote a sequence $\{\epsilon_n^{(k)}\}_{n,k}$
satisfying
\begin{align*}
	\phi^{(k)}(\epsilon_n^{(k)}) \leq n(\epsilon_n^{(k)})^2.
\end{align*}
The order of $\epsilon_n^{(k)}$ depends on a choice of kernel
function, where the optimal minimax rate is $\epsilon_n^{(k)} =
O(n^{-\beta_k / (\beta_k + I_k)})$~\cite{tsybakov2008}. In the
following part, we set $\tilde{\epsilon_n}^{(k)} = \tilde{\epsilon}_n$
for every $k$. As $k' = \arg \max_k I_k$, the $\epsilon_n^{(k)}$
satisfies the condition about the concentration function.

We also note the relation between posterior contraction and the well-known risk bound.  Suppose the
posterior contraction such that $E \Pi_n(\theta : d_n^2(\theta,
\theta_0) \geq C \epsilon_n^2|D_n) \rightarrow 0$ holds, where
$\theta$ is a parameter, $\theta_0$ is a true value, and $d_n$ is a
bounded metric.  Then we obtain the following inequality:
\begin{align*}
	E \Pi_n (d_n^2(\theta, \theta_0)|D_n) \leq C \epsilon_n^2 + DE \Pi_n(\theta : d_n^2(\theta, \theta_0) \geq C \epsilon_n^2|D_n),
\end{align*}
where $D$ is a bound of $d_n$. This leads $E \Pi_n (d_n^2(\theta,
\theta_0)|D_n) = O(C \epsilon_n^2)$.  In addition, if $\theta \mapsto
d_n^2(\theta,\theta_0)$ is convex, the Jensen's inequality provides
$d_n^2(\theta,\theta_0) \leq \Pi_n (d_n^2(\theta, \theta_0)|D_n)$.  By
taking the expectation, we obtain
\begin{align*}
	Ed_n^2(\theta,\theta_0) \leq E \Pi_n (d_n^2(\theta, \theta_0)|D_n) = O(C \epsilon_n^2).
\end{align*}

We start to prove Theorem \ref{thm:converge}. First, we provide a lemma for
functional decomposition.  When the function has a form of a
$k$-product of local functions, we bound a distance between two
functions with a $k$-sum of distance by the local functions.
\begin{lem} \label{lem:decomp_multi} Suppose that two functions
  $f,g:\times_{k=1}^K \mathcal{X}_k \rightarrow \mathbb{R}$ have a
  form $f = \prod_k f_k$ and $g = \prod_k g_k$ with local functions
  $f_k,g_k:\mathcal{X}_k \rightarrow \mathbb{R}$. Then we have a bound
  such that
	\begin{align*}
		\left\| f - g \right\| \leq \sum_k \left\| f_k - g_k \right\| \max\left\{ \left\| f_k \right\|,\left\| g_k \right\| \right\}.
	\end{align*}
\end{lem}
\begin{proof}
	We show the result based on induction. When $k=2$, we have
	\begin{align*}
		f - g = f_1 f_2 - g_1 g_2 = f_1 (f_2  - g_2) + (f_1 - g_1) g_2,
	\end{align*}
	and
	\begin{align*}
		\left\| f - g \right\| \leq \left\| f_1 \right\| \left\| f_2  - g_2 \right\| + \left\| f_1 - g_1 \right\| \left\|   g_2 \right\|.
	\end{align*}
	Thus the result holds when $k=2$.
	
	Assume the result holds when $k = k'$. Let $k = k'+1$. The
        difference between $k'+1$-product functions is written as
	\begin{align*}
		f - g = f_{k'+1} \prod_{k'} f_k - g_{k'+1} \prod_{k'} g_k = \prod_{k'}f_k (f_{k'+1} - g_{k'+1}) + (\prod_{k'} f_k - \prod_{k'} g_k) g_{k'+1}.
	\end{align*}
	From this, we obtain the bound
	\begin{align*}
		\left\| f - g \right\| \leq \left\| \prod_{k'}f_k \right\| \left\| f_{k'+1} - g_{k'+1} \right\| + \left\| \prod_{k'} f_k - \prod_{k'} g_k \right\| \left\|   g_{k'+1} \right\|.
	\end{align*}
	The distance $\left\| \prod_{k'} f_k - \prod_{k'} g_k
        \right\|$ is decomposed recursively by the case of $k =
        k'$. Then we obtain the result.
\end{proof}

Now we provide the proof of Theorem \ref{thm:converge}. Note that $M=M^* < \infty$ in
the statement in Theorem \ref{thm:converge}. Asume that $C,C',C'',\ldots$
are some positive finite
constants and they are not affected by other values.

\begin{proof}
  We will show that $F_x$ satisfies the condition (A).
  Firstly, we check the third condition in the
  theorem.
	
	According to Lemma~\ref{lem:decomp_multi}, the value $\|F_x-f_0\|$ is bounded as
	\begin{align*}
		\|F_x-f_0\| &\leq \sum_m \sum_r \left\| \prod_k W_r^{(k)} - \prod_k w_{m,0}^{(k)} \right\| \\
		&\leq \sum_m \sum_r \sum_k \left\|  W_m^{(k)} - w_{m,0}^{(k)} \right\| \prod_{k' \neq k} \max \left\{ \left\|  w_{m,0}^{(k')} \right\|,\left\|  W_m^{(k')} \right\| \right\}.
	\end{align*}
	By denoting $\left\|\tilde{W}_m^{(k')} \right\| := \max
        \left\{ \left\| w_{m,0}^{(k')} \right\|,\left\| W_m^{(k')}
          \right\| \right\}$, we evaluate the probability $Pr(\|F -
        f_0\|\leq \epsilon_n)$ as
\begin{align}
	\mbox{Pr}(\|F - f_0\|\leq \epsilon_n) &\geq \mbox{Pr}\left(\sum_m \sum_r \sum_k \left\|  W_m^{(k)} -  w_{m,0}^{(k)} \right\| \prod_{k' \neq k}\left\|  \tilde{W}_m^{(k')} \right\|\leq \epsilon_n \right) \notag\\
	&\geq  \mbox{Pr}\left( \sum_r \sum_k \sum_m \left\|  W_m^{(k)} -  w_{m,0}^{(k)} \right\| \leq \frac{1}{C_{k}}\epsilon_n \right), \label{bound_3rd_cond}
\end{align}
where $C_{k}$ is a positive finite constant satisfying $C_k = \max_{m}
\prod_k \left\| \tilde{W}_m^{(k')} \right\|$. 

From \cite{vdVvZ2008}, we use the following inequality for every Gaussian random element $W$:
\begin{align*}
	\mbox{Pr}\left( \|W - w_0\| \leq \epsilon_n \right) \geq \exp(-n\epsilon_n^2),
\end{align*}
Then, by seting $\epsilon_n = \sum_{m,r,k} \epsilon_n^{(k)}$ and with some constant $C$, we bound
(\ref{bound_3rd_cond}) below as
\begin{align*}
	&\mbox{Pr}\left( \sum_r \sum_k \sum_m \left\|  W_m^{(k)} -  w_{m,0}^{(k)} \right\| \leq \frac{1}{C_{k}}\sum_{m,r,k} \epsilon_n^{(k)} \right) \\
	& \geq \prod_{m,r,k} \mbox{Pr}\left( \left\|  W_m^{(k)} -  w_{m,0}^{(k)} \right\| \leq \frac{1}{C_{k}} \epsilon_n^{(k)} \right) \\
	&\geq  \prod_m \prod_r \prod_k \exp\left(-\frac{n}{\left(C_{k}\right)^2 } \epsilon_n^{(k),2}\right) \\
	&\geq  \exp\left(-n\sum_{m,r,k}\epsilon_n^{(k),2}\right).
\end{align*}

For the second condition, we define a subspace of the Banach space as
\begin{align*}
	B_n^{(k)} = \epsilon_n^{(k)} \mathcal{B}_1^{(k)} + M_n^{(k)} \mathcal{H}_1^{(k)},
\end{align*}
for all $k = 1,\ldots,K$. Note $\mathcal{B}_1^{(k)}$ and $\mathcal{H}_1^{(k)}$ are unit balls in $\mathcal{B}$ and $\mathcal{H}$. Also, we define $B_n$ as
\begin{align*}
	B_n := \left\{ w : w = MR\prod_k w_k,  w_k \in B_n^{(k)} ,\forall k\right\}.
\end{align*}

As shown in \cite{vdVvZ2008}, for every $r$ and $k$, 
\begin{align*}
	Pr\left( W_m^{(k)} \notin B_n^{(k)} \right)\leq 1- \Phi(\alpha_n^{(k)} + M_n^{(k)}),
\end{align*}
where $\Phi$ is the cumulative distribution function of the standard
Gaussian distribution;
$\alpha_n^{(k)}$ and $M_n^{(k)}$ satisfy the following equation with a
constant $C' > 0$ as
\begin{align*}
	&\alpha_n^{(k)} = \Phi^{-1}(Pr(W_m^{(k)} \in \epsilon_n \mathcal{B}_1^{(k)})) = \Phi^{-1}(\exp(-\phi_0(\epsilon_n^{(k)}))),\\
	&M_n^{(k)} = -2 \Phi^{-1}(\exp(-C'n(\epsilon_n^{(k)})^2)).
\end{align*}
By setting $\alpha_n^{(k)} + M_n^{(k)} \geq \frac{1}{2}M_n^{(k)}$ and
using the relation $\phi_0(\epsilon) \leq n \epsilon_n^2$, we have
\begin{align*}
	Pr\left( W_m^{(k)} \notin B_n^{(k)} \right)\leq 1- \Phi\left(\frac{1}{2}M_n^{(k)}\right) = \exp(-C'n(\epsilon_n^{(k)})^2).
\end{align*}
This leads
\begin{align*}
	Pr(F_x \notin B_n) &\leq \prod_m \prod_k \prod_r \mbox{Pr}( W_m^{(k)} \notin B_n^{(k)}) \\
	&\leq \prod_m \prod_k \prod_r \exp(-C'n(\epsilon_n^{(k)})^2)\\
	&= \exp(-C'\sum_{m,r,k}(\epsilon_n^{(k)})^2).
\end{align*}

Finally, we show the first condition. Let
$\{h_j^{(k)}\}_{j=1}^{N^{(k)}}$ be a set of elements of $M_n^{(k)}
\mathcal{H}_1^{(k)}$ for all $k$. Also, we set that each $h_j^{(k)}$
are $2 \epsilon_n^{(k)}$ separated, thus $\epsilon_n^{(k)}$ balls
with center $h_j^{(k)}$ do not have intersections.
According to Section 5 in \cite{vdVvZ2008}, we have
\begin{align*}
	1 &\geq \sum_{j=1}^{N^{(k)}}\mbox{Pr}(W_m^{(k)} \in h_j^{(k)} + \epsilon_n B_1^{(k)}) \\
	& \geq \sum_{j=1}^{N^{(k)}} \exp(-\frac{1}{2}\|h_j^{(k)}\|_{\mathcal{H}}^2)Pr(W \in \epsilon_n^{(k)} \mathcal{B}_1^{(k)}) \\
	& \geq N^{(k)} \exp\left(-\frac{1}{2}(M_n^{(k)})^2\right)\exp(-\phi_0(\epsilon_n^{(k)})).
\end{align*}
Consider $2 \epsilon_n^{(k)}$-nets with center $\{h_j^{(k)}\}_{j=1}^{N^{(k)}}$.
The nets cover $M_n^{(k)} \mathcal{H}_1^{(k)}$, we obtain
\begin{align*}
	N(2 \epsilon_n^{(k)}, M_n^{(k)} \mathcal{H}_1^{(k)},\|\cdot\|) \leq N^{(k)} \leq \exp\left(\frac{1}{2}(M_n^{(k)})^2\right)\exp(\phi^{(k)}(\epsilon_n^{(k)})).
\end{align*}
Because every point in $B_n^{(k)}$ is within $\epsilon_n^{(k)}$ from some point of $M_n \mathcal{H}_1^{(k)}$, we have
\begin{align*}
	N(3 \epsilon_n^{(k)}, B_n^{(k)},\|\cdot\|) \leq N(2 \epsilon_n^{(k)}, M_n^{(k)} \mathcal{H}_1^{(k)},\|\cdot\|).
\end{align*}

By Lemma \ref{lem:decomp_multi}, for every elements $w,w' \in B_n$
constructed as $w = \prod _k w^{(k)}, w^{(k)} \in B_n^{(k)}$, its
distance is evaluated as
\begin{align}
	\|w - w'\| &= \left\| \prod_k w_k  -  \prod_k w'_k\right\| \notag\\
	& \leq \sum_k \left\|w^{(k)} - w'^{(k)}\right\| \prod_{k' \neq k}\left\|\tilde{w}^{(k')}\right\| \notag\\
	& \leq \prod_{k' \neq k} C_{k'} \sum_k \left\|w^{(k)} - w'^{(k)}\right\|. \label{ineq1}
\end{align}
We consider a set $\{h^*: h^* = \prod_{k,j} h_j^{(k)} \}$, which are the element of $B_n$.
According to \eqref{ineq1},
the $C \epsilon_n$-net with center $\{h^*\}$ will cover $B_n$, and its
number is equal to $\prod_k N(\epsilon_n,B_n^{(k)},\|\cdot\|)$. Let
$C\sum_{m,r,k} \epsilon_n^{(k)} =: \epsilon_n'$ and we have
\begin{align*}
	\log N(3\epsilon'_n,B_n,\|\cdot\|) &\leq \sum_{m,r,k} \log N(3\epsilon_n^{(k)},B_n^{(k)},\|\cdot\|) \\
	& \leq \sum_{m,r,k} \log N(2 \epsilon^{(k)}_n, M_n^{(k)} \mathcal{H}_1^{(k)},\|\cdot\|) \\
	& \leq \sum_{m,r,k} \left( \frac{1}{2}(M_n^{(k)})^2 + \phi^{(k)}(\epsilon^{(k)}_n) \right) \\
	& \leq \sum_{m,r,k} \left( C'' n (\epsilon^{(k)}_n)^2 + C''' n (\epsilon_n^{(k)})^2 \right) \\
	& \leq C'''' n\sum_{m,r,k} (\epsilon^{(k)}_n)^2.
\end{align*}
The last inequality is from the definition of $M_n$ and
$\phi^{(k)}(\epsilon_n)$. 

We check that the conditions (A) are all satisfied, thus we obtain
posterior contraction of the GP estimator with rate $\epsilon^{(k)}_n$.
Also, according to the connection between posterior contraction and the risk bound,
we achieve the result of Theorem \ref{thm:converge}.


\end{proof}


\section{Proof of Theorem \ref{thm:high_model}}

We define the representation for the true function.  Recall the
notation $f = \sum_m \bar{f}_m$. We introduce the notation for 
the true function $f^*$ and the GP estimator $\hat{f}$ as follow:
\begin{align*}
  f^* = \sum_{m=1}^{\infty} \bar{f}^*_m \\
  \hat{f}_n = \sum_{m=1}^M \hat{\bar{f}}_m.
\end{align*}

We decompose the above two functions as follows:
\begin{align}
	\|f^* - \hat{f}\|_n &= \left\| \sum_{m=1}^{\infty} \bar{f}^*_m - \sum_{m=1}^M \hat{\bar{f}}_m\right\|_n \notag\\
	&=  \left\| \sum_{m=M+1}^{\infty} \bar{f}_m^* - \sum_{m=1}^M (\bar{f}_m^*  - \hat{\bar{f}}_m)\right\|_n \notag\\
	& \leq \left\| \sum_{m=M+1}^{\infty} \bar{f}_m^* \right\|_n + \left\| \sum_{m=1}^M (\bar{f}_m^*  - \hat{\bar{f}}_m)\right\|_n. \label{ineq2}
\end{align}

Consider the first term with Assumption 1.
The expectation of the first term of \eqref{ineq2} is bounded by
\begin{align*}
	&E\left\| \sum_{m=M+1}^{\infty} \bar{f}_m^* \right\|_n \leq \sum_{m=M+1}^{\infty}\left\| \bar{f}_m^* \right\|_2 \\
	&\leq C \sum_{m=M+1}^{\infty} m^{-\gamma},
\end{align*}
with finite constant $C$.  The exchangeability of the first inequality
is guaranteed by the setting of $f^*$.  Also, the second inequality
comes from Assumption 1.  Then, the infinite summation of
$m^{-\gamma}$ enables us to obtain that the expectation of the first
term is $O(M^{-\gamma})$, by using the relation $\sum_{x=C}^{\infty}x^{-r} \asymp C^{-r}$.

About the second term of \eqref{ineq2}, we consider the estimation for $f^*$ with finite $M$.  As
shown in the proof of Theorem \ref{thm:converge}, the estimation of $\sum_m^M \bar{f}_m$ is
evaluated as
\begin{align*}
	E \|\hat{f} - f^*\|_2 = O\left(\sum_m^M n^{-\frac{\beta}{2\beta + \max_k I_k}} \right) =  O\left(M n^{-\frac{\beta}{2\beta + \max_k I_k}} \right).
\end{align*}

Finally, we obtain the relation
\begin{align*}
	E\|f^* - \hat{f}\|_n = O(M^{-\gamma}) + O\left(M n^{-\frac{\beta}{2\beta + \max_k I_k}} \right).
\end{align*}
Then, we allow $M$ to increase as $n$ increases.  Let $M \asymp
n^{\zeta}$ with positive constant $\zeta$, and simple calculation
concludes 
that $\zeta = (\frac{\beta}{2\beta
  + \max_k I_k})/(1+\gamma)$ is optimal.  By substituting $\zeta$, we
obtain the result.


\newpage
\bibliography{./bib_master}
\end{document}